\def\year{2019}\relax
\documentclass{article}
\usepackage{ijcai19}
\usepackage{times}
\usepackage{soul}
\usepackage{url}
\usepackage[hidelinks]{hyperref}
\usepackage[utf8]{inputenc}
\usepackage[small]{caption}
\usepackage{graphicx}
\usepackage{amsmath}
\usepackage{booktabs}
\usepackage{algorithm}
\usepackage{algorithmic}
\usepackage{times}  %Required
\usepackage{helvet}  %Required
\usepackage{courier}  %Required
\usepackage{url}  %Required
\usepackage{graphicx}  %Required
\usepackage{mathtools, cuted}
\usepackage{amsmath}
\usepackage{amssymb}
\usepackage{subcaption}
\usepackage{amsfonts}
\usepackage{amsthm}
\usepackage{mathtools}

\newtheorem{theorem}{Theorem}[section]

\newtheorem{definition}{Definition}[section]

\DeclarePairedDelimiter{\ceil}{\lceil}{\rceil}

\usepackage{mathletters}
\DeclareMathOperator*{\argmax}{arg\,max}

\usepackage{authblk}
\usepackage{mdframed}
\usepackage{todonotes}

\usepackage{soul}
\begin{document}
\title{Learning Modular Safe Policies in the Bandit Setting with Application to Adaptive Clinical Trials
}
\author[1]{Hossein Aboutalebi\footnote{hossein.aboutalebi@mail.mcgill.ca}}
\author[1]{Doina Precup}
\author[2]{ Tibor Schuster}
\affil[1]{Department of Computer Science, McGill University. Mila Quebec AI Institute}
\affil[2]{Department of Family Medicine, McGill University}
\maketitle

\begin{abstract}
	The stochastic multi-armed bandit problem is a
	well-known model for studying the exploration-exploitation trade-off. It has significant possible applications in adaptive clinical trials, which allow
	for dynamic changes in the treatment allocation probabilities of patients.
	However, most bandit learning algorithms are designed with the goal of minimizing the expected
	regret. While this approach is useful in many areas,
	in clinical trials, it can be sensitive to outlier data, especially when the sample size is small. In this paper, we define and study a new robustness criterion for bandit problems. Specifically, we consider optimizing a function
	of the distribution of returns as a regret measure. This provides practitioners more flexibility to
	define an appropriate regret measure. The learning algorithm we propose to solve this type of problem is a modification of the
	BESA algorithm \cite{Baransi2014}, which considers a more general version of regret. We present a regret bound
	for our approach and evaluate it empirically both
	on synthetic problems as well as on a dataset from
	the clinical trial literature. Our approach compares
	favorably to a suite of standard bandit algorithms. Finally, we provide a web application where users can create their desired synthetic bandit environment and compare the performance of different bandit algorithms online.
\end{abstract}

\section{Introduction}

The multi-armed bandit is a standard model for researchers to investigate the exploration-exploitation trade-off, see e.g~\cite{Baransi2014,auer2002finite,sani2012risk,chapelle2011empirical,sutton1998reinforcement}. One of the main advantage of multi-armed bandit problems is its simplicity that allows for a higher level of theoretical studies.

%In this kind of problem, we seek to identify the best policy while making sure we have investigated all the other possible policies available for a given problem. More specifically, t

The multi-armed bandit problem consists of a set of arms, each of which generates a stochastic reward from a fixed but unknown distribution associated to it. Consider a series of mulitple arm pulls (or steps) $t=1,...,T$ and selecting a specific arm  $a \in \cA$ at each step i.e. $a(t)=a_t$.  The standard goal in the multi-armed bandit setting is to find the arm $\star$ which has the maximum expected reward $\mu_\star$ (or equivalently, minimum expected regret). The expected regret after $T$ steps $R_T$ is defined as the sum of the expected difference between the mean reward under $\{a_t\}$ and the reward expected under the optimal arm $\star$:
$$R_T=\mathbb{E}\left[\sum_{t=1}^T(\mu_\star-\mu_{a_t})\right]$$
%Where $\mu_{a_t}$ is the mean reward of the arm $a$ chosen at  step $t$ i.e., $a_t=a(t)$. 

While this objective is very popular, there are practical applications, for example in medical research and AI safety~\cite{garcia2015comprehensive} where maximizing expected value is not sufficient, and it would be better to have an algorithm sensitive also to the variability of the outcomes of a given arm. 
For example, consider multi-arm clinical trials where the objective is
to find the most promising treatment among a pool of available
treatments. Due to heterogeneity in patients'  treatment responses, considering only the expected mean may not be of interest~\cite{austin2011introduction}. Specifically, as
the mean is usually sensitive to outliers and does not provide information about the dispersion of individual responses, the
expected reward has only limited value in achieving a clinical trial's objective. Due to these problems, previous contributions like \cite{sani2012risk} try to include the variance of rewards in the regret definition and develop algorithms to solve this slightly enhanced problem. While these modified approaches try to consider variablity in the response of arms, they induce new problems due to the fact that the variance is not necessarily a good measure of variablity for a distribution. This is because the variance equally penalizes responses that are above or below the mean response.  Other articles like \cite{galichet2013exploration} try to use the conditional value at risk to define a better regret definition. Though the conditional value at risk may address the problem we faced with including variance, it may not reflect the amount of variablity we could observe for a distribution over its entire domain. All in all, the consistency
of treatments among patients is essential, with the ideal treatment
usually defined as the one which has a high
positive response rate while showing low variability in response among
patients. Thus, the idea of consistency and saftey seems to some extent subjective and problem dependant. As a result, it might be necessary to develop an algorithm which can work with an arbitrary definition of consistency for a distribution.

This kind of system design which allows the separation of different parts of a system (here regret function and learning algorithm) has already been explored in modular programming. In modular programming, we emphasize on splitting the entire system into independant modules which at the end, the composite of these modules builds our system. This design trick is necessary when we are dealing with the change of customer demands and we require our system to adapt with the new demands. Here, we follow the same paradigm by making regret definition independent of the learning algorithm. As a result, we allow more flexibility in defining the regret function which is capable of incorporating problem specific demands.

Finally, we achieve the aforementioned goals by  extending one of the recent  algorithms in the bandit literature called BESA (Best Empirical Sampled Average)~\cite{Baransi2014}. One of the main advantage of BESA compared to other existing bandit algorithms  is that it does not involve many hyper-parameters. This  is especially useful when one does not have any prior knowledge or has insufficient prior knowledge about the different arms in the beginning. Also, this feature makes it easier to introduce modular design by using McDiarmid's Lemma \cite{el2009transductive}. 

\textbf{Key contributions:} We provide a modular definition of regret called safety-aware regret which allows higher flexibility in defining the risk for multi-armed bandit problems. We propose a new algorithm called BESA+ which solves this category of problems. We show the upper-bounds of its safety-aware regret for two-armed and multi-armed bandits. For the experiment parts, we compare our model with some of the notable earlier research works and show that BESA+ has a satisfying performance. For the last experiment, we depict the performance of our algorithm on a real clinical dataset and illustrate that it is capable of solving the problem with user-defined safety-aware regret. Finally, for the first time as far as we know, we provide a web application which allows users to create their own custom environment and compare our algorithm with other works. 
%\begin{itemize}
%    \item Consistency: ``the quality of always behaving or performing in a similar way, or of always happening in a similar way''~\footnote{\url{https://dictionary.cambridge.org/us/dictionary/english/consistency}}
%    \item Some applications (e.g. healthcare?) require both good and consistent outcomes
%    \item Tradeoff between expected outcome and statistical dispersion
%\end{itemize}
\section{Background and Notation}

We consider the standard bandit setting with action (arm) set~$\cA$, where each action $a\in\cA$ is characterized by a reward distribution $\phi_a$. The distribution for action $a$ has mean $\mu_a$ and variance $\sigma_a^2$.  Let $X_{a,i} \sim \phi_a$ denote the $i$-th reward sampled from the distribution of action $a$. All actions and samples are independent. The bandit problem is described as an iterative game where, on each step (round) $t$, the player (an algorithm) selects action (arm) $a_t$ and observes sample $X_{a,N_{a,t}}$, where $N_{a,t} = \sum_{s=1}^t \indic{a_s=a}$ denotes the number of samples observed for action~$a$ up to time $t$ (inclusively). A policy is a distribution over $\cA$.  In general, stochastic distributions are necessary during the learning stage, in order to identify the best arm. We discuss the exact notion of ``best" below.

%Also consider the set $n=\{1, 2, ..., m\}$.
 We define  $I_S(m,j)$ as the set obtained by sub-sampling without replacement $j$ elements form the set $S$ of size $m$. Let $\cX_{a,t}$ denote the history of observations (records) obtained from action (arm) $a$ up to time $t$ (inclusively), such that $|\cX_{a,t}| = N_{a,t}$. The notation $\cX_{a,t}(\cI)$ indicates the set of sub-samples from $\cX_{a,t}$, where sub-sample  $\cI \subset \{ 1, 2, \dots, N_{a,t}\}$.

The multi-armed bandit was first presented in the seminal work of Robbins~\cite{robbins1985some}. It has been shown that under certain conditions~\cite{burnetas1996optimal,lai1985asymptotically}, a policy can have logarithmic cumulative regret: 
$$\lim_{t\rightarrow \infty}\inf \frac{\kR_t}{\log(t)}\geq \sum_{a:\mu_a<\mu_\star}\frac{\mu_\star-\mu_a}{K_{\inf}(r_a;r_\star)}$$
where $K_{\inf}(r_a;r_\star)$ is the Kullback-Leibler divergence between the reward distributions of the respective arms. Policies for which this bound holds are called {\em admissible}.

Several algorithms have been shown to produce admissible policies, including UCB1~\cite{auer2002finite}, Thompson sampling \cite{chapelle2011empirical,agrawal2013further} and BESA \cite{Baransi2014}.  However, theoretical bounds are not always matched by empirical results. For example, it has been shown in~\cite{kuleshov2014algorithms} that two algorithms which do not produce admissible policies, $\epsilon$-greedy and Boltzmann exploration~\cite{sutton1998reinforcement}, behave better than UCB1 on certain problems. Both BESA and Thompson sampling were shown to have comparable performance with Softmax and $\epsilon$-greedy.

% Thompson sampling is based on the idea of Bayesian methods which define a prior distribution for each arm and then update the posterior distribution based on the new observed rewards. We will elaborate on the  BESA algorithm in the next section.
%In this regard, many research has been conducted to develop learning algorithms which minimize the expected regret. For example, $\epsilon$-greedy and Softmax algorithms are among the most popular ones \cite{sutton1998reinforcement}. like  

While the expected regret is a natural and popular measure of performance which allows the development of theoretical results, recently, some papers have explored other definitions for regret. For example,~\cite{Sani2012} consider a linear combination of variance and mean as the definition of regret for a learning algorithm $A$:
\begin{align} \label{amir_regret}
	\hat{MV}_t(A)=\hat{ \sigma}_t^2(A)-\rho\hat{\mu}_t(A)
\end{align}
where $\hat{\mu}_t$ is the estimate of the average of observed rewards up to time step $t$ and $\hat{\sigma}_t$ is a biased estimate of the variance of rewards up to time step $t$. The regret is then defined as:
$$\kR_t(A)=\hat{MV}_t(A)-\hat{MV}_{\star,t}(A),$$ 
 where $\star$ is the optimal arm. 
 According to~\cite{Maillard2013}, however, this definition is going to penalize the algorithm if it switches between optimal arms.
Instead, in~\cite{Maillard2013}, the authors devise a new definition of regret which controls the lower tail of the reward distribution. However, the algorithm to solve the corresponding objective function seems time-consuming, and the optimization to be performed may be intricate. Finally, in \cite{galichet2013exploration}, the authors use the notion of conditional value at risk in order to define the regret.

\section{Measure of regret}

Unlike previous works, we now give a formal definition of class of functions which can be used as a separate module inside our learning algorithm module to measure the regret. We call these class of functions "safety value functions". 

In the following section, we try to formally define these functions. Assume we have $k$ arms ($|\cA|=k$) with reward distributions $\phi_1, \phi_2, \dots, \phi_k$.

\begin{definition} \textbf{safety value function:}  Let $\cD$ denotes the set of all possible reward distributions for a given interval. The safety value function $v: \cD \rightarrow \cR$ provides a score for a given distribution. 
	
The optimal arm $\star$ under this value function is defined as	  
	\begin{align}
	\star \in \arg\max_{a\in\cA} (v(\phi_a))
	\end{align}
	
The regret corresponding to the safety value function up to time $T$ is defined as:
 	\begin{align} \label{my_regret}
 	\kR_{T,v}=\mathbb{E}\left[\sum_{t=1}^T(v(\phi_\star)-v({\phi}_{a_t}))\right]
 	\end{align}
We call \eqref{my_regret}, safety-aware regret. 	
\end{definition}
When the context is clear, we usually drop the subscript $v$ and use only $\kR_T$ for the ease of notation.  
\begin{definition} \textbf{Well-behaved safety value function:}  Given a reward distribution $\phi_a$ over the interval $[0,1]$, a safety value function $v$ for this distribution is called well-behaved if there exists an unbiased estimator $\hat{ v}$ of $v$ such that for any set of observation $\{x_1, x_2, \dots, x_n\}$ sampled from $\phi_a$, and for some constant $\gamma$ we have:
%	\begin{strip}
	\begin{align} \label{my_condition}
	\sup_{\hat{ x_i}}|\hat{ v}(x_1,\dots,x_i,\dots,x_n)-\hat{ v}(x_1,\dots,\hat{x_i},\dots,x_n)|<\frac{\gamma}{n}
	\end{align}
%\end{strip}
If \eqref{my_condition} holds for any reward distribution $\phi$ over the interval $[0,1]$, we call the safety value function $v$, a well-behaved safety value function.
\end{definition}

\textbf{Example 1:} \label{my_example} For a given arm $a$ which has reward distribution limited to interval $[0,1]$, consider the safety value function $\mu_a-\rho\sigma_a^2$ which measures the balance between the mean and the variance of the reward distribution of arm $a$. $\rho$ is a hyper-parameter constant for adjusting the balance between variance and the mean. This is a well-behaved safety function if we use the following estimator for computing empirical mean and variance:

\begin{eqnarray}
\hat{\mu}_{a,t}&=&\frac{1}{N_{a,t}}\sum_{i=1}^{N_{a,t}}r_{a,i} \label{eqn:empirical_mean}\\
\hat{\sigma}_{a,t}^2&=&\frac{1}{N_{a,t}-1}\sum_{i=1}^{N_{a,t}}(r_{a,i}-\hat{\mu}_{a,t})^2
\end{eqnarray}

where $r_{a,i}$ is the $i$th reward obtained from pulling arm $a$. It should be clear that the unbiased estimator $\hat{\mu}_{a,t}-\rho\hat{\sigma}_{a,t}^2$ satisfies \eqref{my_condition}. $\square$

Other types of well-behaved safety function can be defined as a function of standard deviation or conditional value at risk similar to the previous example. In the next section, we are going to develop an algorithm which can optimize the safety-aware regret.

\section{Proposed Algorithm}

%\subsection{Overview of the original BESA algorithm}

In order to optimize the safety-aware regret, we build on the BESA algorithm, which we will now briefly review.
As discussed in~\cite{Baransi2014}, BESA is a non-parametric (without hyperparameter) approach for finding the optimal arm according to the expected mean regret criterion. Consider a two-armed bandit with actions $a$ and $\star$ ,where $\mu_{\star} >\mu_{a}$, and assume that $N_{a,t}<N_{\star,t}$ at time step $t$. In order to select the next arm for time step $t+1$, BESA first sub-samples $s_\star=I_\star(N_{\star,t},N_{a,t})$ from the observation history (records) of the arm $\star$  and similarly sub-sample  $s_a=I_a(N_{a,t},N_{a,t})=\cX_{a,t}$ from the records of  arm $a$. If $\hat{\mu}_{s_a}>\hat{\mu}_{s_{\star}}$, BESA chooses arm $a$, otherwise it chooses arm $\star$. 

The main reason behind the sub-sampling  is that it gives a similar opportunity to both arms. Consequently,  the effect of having a small sample size,  which may cause bias in the estimates diminishes. When there are more than two arms, BESA runs a tournament algorithm on the arms~\cite{Baransi2014}. 

Finally, it is worth mentioning that the proof of the regret bound of BESA uses a non-trivial lemma for which authors did not provide any formal proof. In this paper, we will avoid using this lemma to prove the soundness of our proposed algorithm for a more general regret family. Also, we extend the proof for the multi-armed case which was not provided in the \cite{Baransi2014}.

We are now ready to outline our proposed approach, which we call BESA+. As in~\cite{Baransi2014}, we focus on the two-arm bandit.
For more than two arms, a tournament can be set up in our case as well.

\begin{algorithm}
\caption{\textbf{BESA+} ~ two action case}
	\textbf{Input}: Safety aware value function $v$ and its estimate $\hat{v}$
    
    \textbf{Parameters}: current time step $t$, actions $a$ and $b$. Initially $ N_{a,0}=0,  N_{b,0}=0$
    
    \begin{algorithmic}[1]
        \IF{$N_{a,t-1} = 0\vee N_{a,t-1} < \log(t)$}
            \STATE $a_t = a$
        \ELSIF{$N_{b,t-1} = 0 \vee N_{b,t-1} < \log(t)$}
            \STATE $a_t = b$
        \ELSE
            \STATE $n_{t-1} = \min\{N_{a,t-1} , N_{b,t-1}\}$
            \STATE $\cI_{a,t-1} \leftarrow I_{a}(N_{a,t-1},n_{t-1})$
            \STATE $\cI_{b,t-1} \leftarrow I_{b}(N_{b,t-1},n_{t-1})$
            \STATE Calculate $\tilde v_{a,t} = \hat v(\cX_{a,t-1}(\cI_{a,t-1}))$ and $\tilde v_{b,t} = \hat v(\cX_{b,t-1}(\cI_{b,t-1}))$
            \STATE $a_t = \argmax_{i \in \{ a, b \}} \tilde v_{i,t}$~~(break ties by choosing arm with fewer tries)
        \ENDIF
        \RETURN $a_t$
    \end{algorithmic}
    
\label{alg:besa_two_actions}
\end{algorithm}

%\subsection{Comparison of BESA and BESA+ algorithms}
The first major difference between BESA+ and BESA is the use of the safety-aware value function instead of the simple regret.
A second important change is that BESA+ selects the arm which has been tried less up to  time step $t$ if the arm has been chosen less than $\log (t)$ times up to $t$. Essentially, this change in the algorithm is negligible in terms of establishing the total expected regret, as we cannot achieve any better bound than $\log (T)$ which is shown in  Robbins' lemma  \cite{lai1985asymptotically}. This tweak also turns out to be vital  in proving that the expected regret of the BESA+ algorithm is bounded by $\log (T)$ (a result which we present shortly).

To better understand why this modification is necessary,  consider a two arms scenario. The first arm gives a deterministic reward of $r \in [0,0.5)$ and the second arm has a uniform distribution in the interval [0,1] with the expected reward of 0.5. If we are only interested in the expected reward $(\mu)$, the algorithm should ultimately favor the second arm. On the other hand, there exists a probability of $r$ that the BESA algorithm is going to constantly choose the first arm if the second arm gives a value less than $r$ on its first pull. 
In contrast, BESA+ evades  this problem by letting the second arm be selected enough times such that it eventually becomes distinguishable from the first arm. 

We are now ready to state the main theoretical result of our proposed algorithm.

\begin{theorem} \label{my_tehorem}
Let $v$ be a well-behaved safety value function. Assume $\cA=\{a,\star\}$ be a two-armed bandit with bounded rewards $\in [0,1]$, and the value gap $\Delta=v_\star-v_a$. Given the value $\gamma$, the expected safety-aware regret of the Algorithm BESA+ up to time $T$ is upper bounded as follows:
\begin{align} \label{final_regret_bound_1}
\kR_T \leq  \zeta_{\Delta, \gamma}\log(T) + \theta_{\Delta, \gamma}
\end{align}
where in~\eqref{final_regret_bound_1}, $\zeta_{\Delta, \gamma}, \theta_{\Delta, \gamma}$ are constants which are dependent on the value of $\gamma, \Delta$.
\end{theorem} 
\begin{proof}
	Due to the page limit, we could not include all the proof. Here, we just provide a short overview of the proof. The proof mainly consists of two parts. The first part of our proof is similar to \cite{Baransi2014} but instead we have used McDiarmid’s Lemma \cite{el2009transductive} \cite{tolstikhin2017concentration}. For the second part of the proof, unlike \cite{Baransi2014}, we have avoided using the unproven lemma in their work and instead tried to compute the upper bound directly by exploiting the $log$ trick in our algorithm (this trick has been further elaborated in the first experiment). Interested reader can visit \href{https://drive.google.com/file/d/1DIV0ciTKFNNKNzt--dz_Mw-pOGD8YQKX/view}{\textbf{here}} to see the full proof.
\end{proof}

\begin{theorem}
	Let $v$ be a well-behaved safety value function. Assume $\cA=\{a_1,\dots, a_{k-1}, \star\}$ be a k-armed bandit with bounded rewards $\in [0,1]$. Without loss of generality, consider the optimal arm is $\star$ and the value gap for arm $a, \star$ is $\Delta_a=v_\star-v_a$. Also consider  $\Delta_{max}=\max_{a\in \cA} \Delta_a $. Given the value $\gamma$, the expected safety-aware regret of the Algorithm BESA+ up to time $T$ is upper bounded as follows:
	\begin{align} \label{final_regret_bound}
	\kR_T \leq \frac{\Delta_{\max}\ceil{\log k}}{\Delta_{\hat{ a}}}\left[ \zeta_{\Delta_{\hat{ a}}, \gamma}\log(T)+ \theta_{\Delta_{\hat{ a}}, \gamma} \right] + k \Delta_{\max} n
	\end{align}
	where in~\eqref{final_regret_bound}, $\zeta, \theta$ are constants which are dependent on the value of $\gamma, \Delta$. Moreover, $\hat{a}$ is defined:
	$$\hat{a}=\argmax_{a\in \cA}  \zeta_{\Delta_a, \gamma}\log(T) + \theta_{\Delta_a, \gamma}$$
	for $T\geq n$.
\end{theorem} 
\begin{proof}
	We Know that the arm $\star$ has to play at most $\ceil{\log k}$ matches (games) in order to win the round. If it losses any of these $\ceil{\log k}$ games, we know that at that round we will see a regret. This regret should be less than or equal to $\Delta_{max}$. 
	
	In the following,We use notation $\pmb{1}_{-a_\star,i }$  to denote the indicator for the event of $a_\star$ losing the $i$th match ($1\leq i\leq\ceil{\log k}$).
	\begin{align}\label{my_ineq_f}
		\kR_T&= \sum_{t =1}^{T} \sum_{i =1}^{k} \Delta_{a_i} \mathbb{E} [\pmb{1}_{a_t=a_i}] \nonumber\\
		&\leq \sum_{t =1}^{T} \sum_{i =1}^{\ceil{\log k} } \Delta_{\max} \mathbb{E} [\pmb{1}_{-a_\star,i }]\nonumber\\
		&\leq \sum_{t =1}^{T} \sum_{i =1}^{\ceil{\log k} } \Delta_{\max} \max_{i'}\{ \mathbb{E} [\pmb{1}_{-a_\star,i' }]\}\nonumber\\
			&\leq \sum_{i =1}^{\ceil{\log k} } \Delta_{\max} \sum_{t =1}^{T}\max_{i'}\{ \mathbb{E} [\pmb{1}_{-a_\star,i' }]\}\nonumber\\
 &\leq\frac{\Delta_{\max}\ceil{\log k}}{\Delta_{\hat{ a}}}\sum_{t =n}^{T}  \Delta_{\hat{ a}} \mathbb{E} [\pmb{1}_{-a_\star,\hat{a}}]+k \Delta_{\max} n\nonumber\\
 &\leq\frac{\Delta_{\max}\ceil{\log k}}{\Delta_{\hat{ a}}}\left[ \zeta_{\Delta_{\hat{ a}}, \gamma}\log(T)+ \theta_{\Delta_{\hat{ a}}, \gamma} \right]+k\Delta_{\max} n
	\end{align}
	
\end{proof}

% !TeX root = IJCAIpaper.tex

% !TeX root = IJCAIpaper.tex

\section{Empirical results}

\subsection{Empirical comparison of BESA and BESA+}

As discussed in the previous section, BESA+ has some advantages over BESA.  We illustrate the example we discussed in the previous section through the results 
in Figures 1-3, for $r\in \{0.2,0.3,0.4\}$. Each experiment has been repeated 200 times. Note that while BESA has an almost a linear regret behavior, BESA+ can learn the optimal arm within the given  time horizon and its expected accumulated regret is upper bounded by a log function. It is also easy to notice that BESA+ has a faster convergence rate compared with BESA.
As $r$ gets  closer to $0.5$, the problem becomes harder. This phenomenon is a direct illustration of our theoretical result.

\begin{figure}[!htb]
	\centering

	\includegraphics[width=0.7\linewidth]{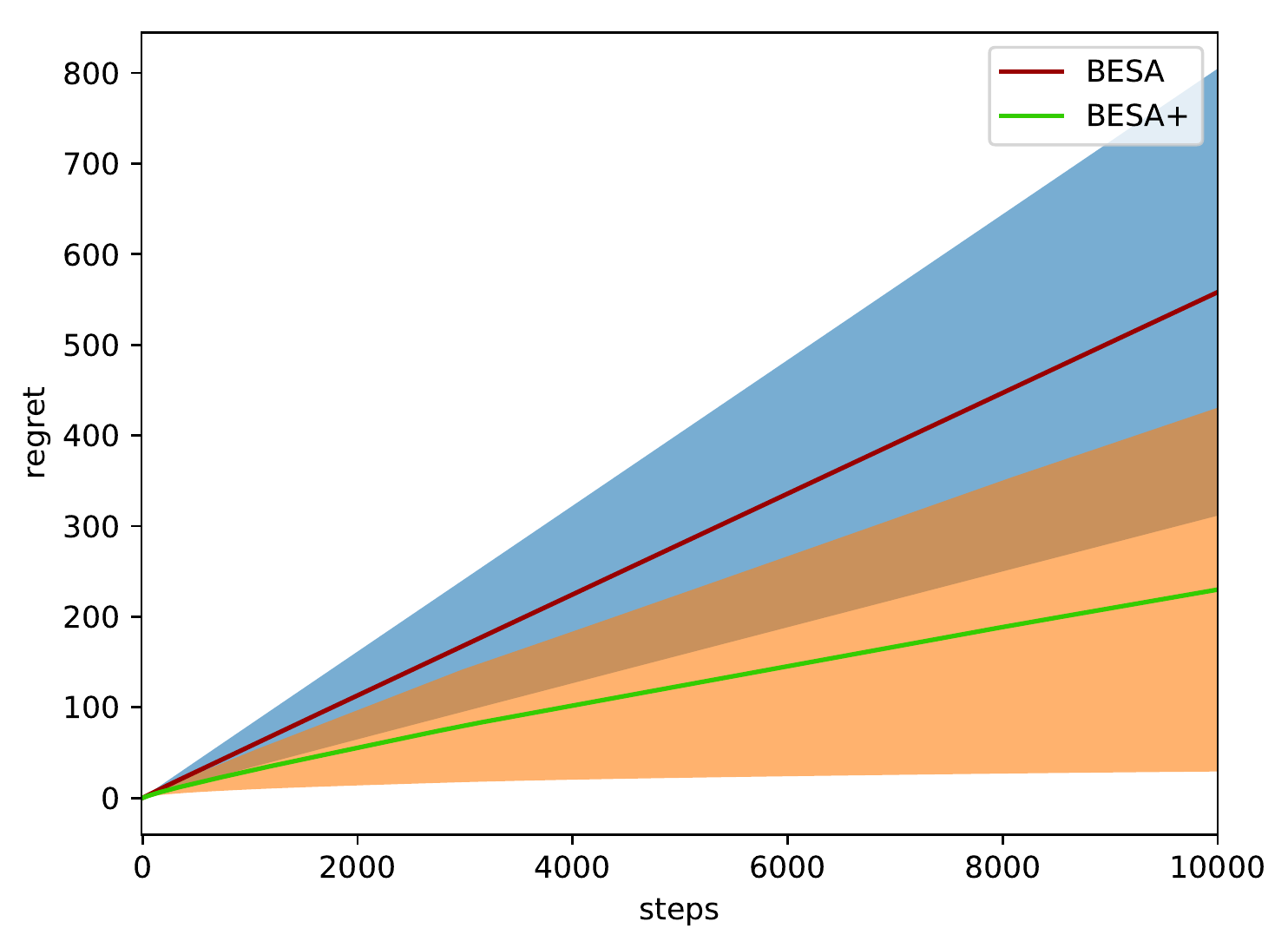}
	\caption{Result of accumulated expected regret for $r=0.4$}
	\label{Fig:Data2}
\end{figure}

\begin{figure}[!htb]
	\centering

	\includegraphics[width=0.7\linewidth]{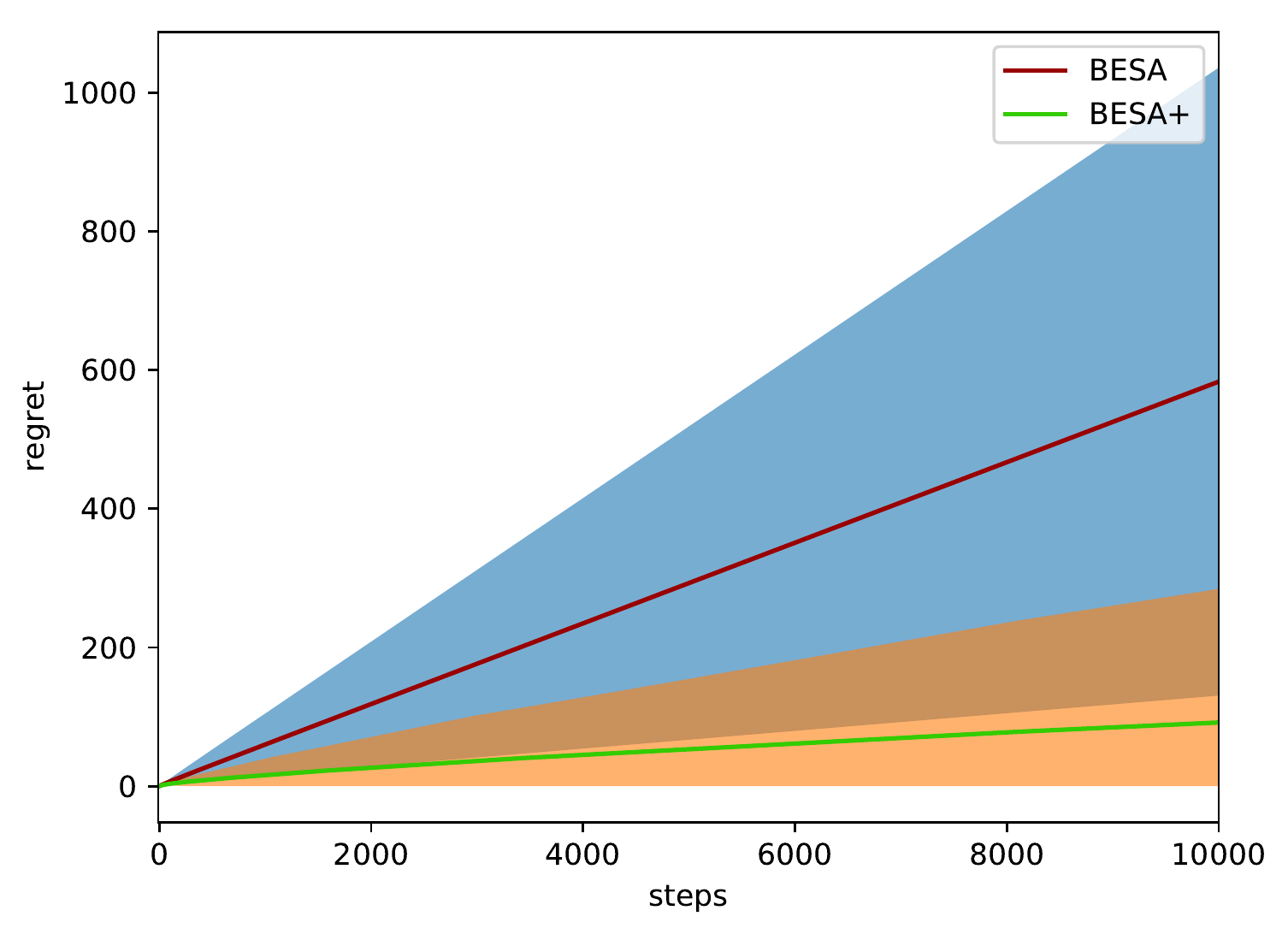}
	\caption{Result of accumulated expected regret for $r=0.3$}\label{Fig:Data2}
\end{figure}

\begin{figure}[!htb]
	\centering

	\centering
	\includegraphics[width=0.7\linewidth]{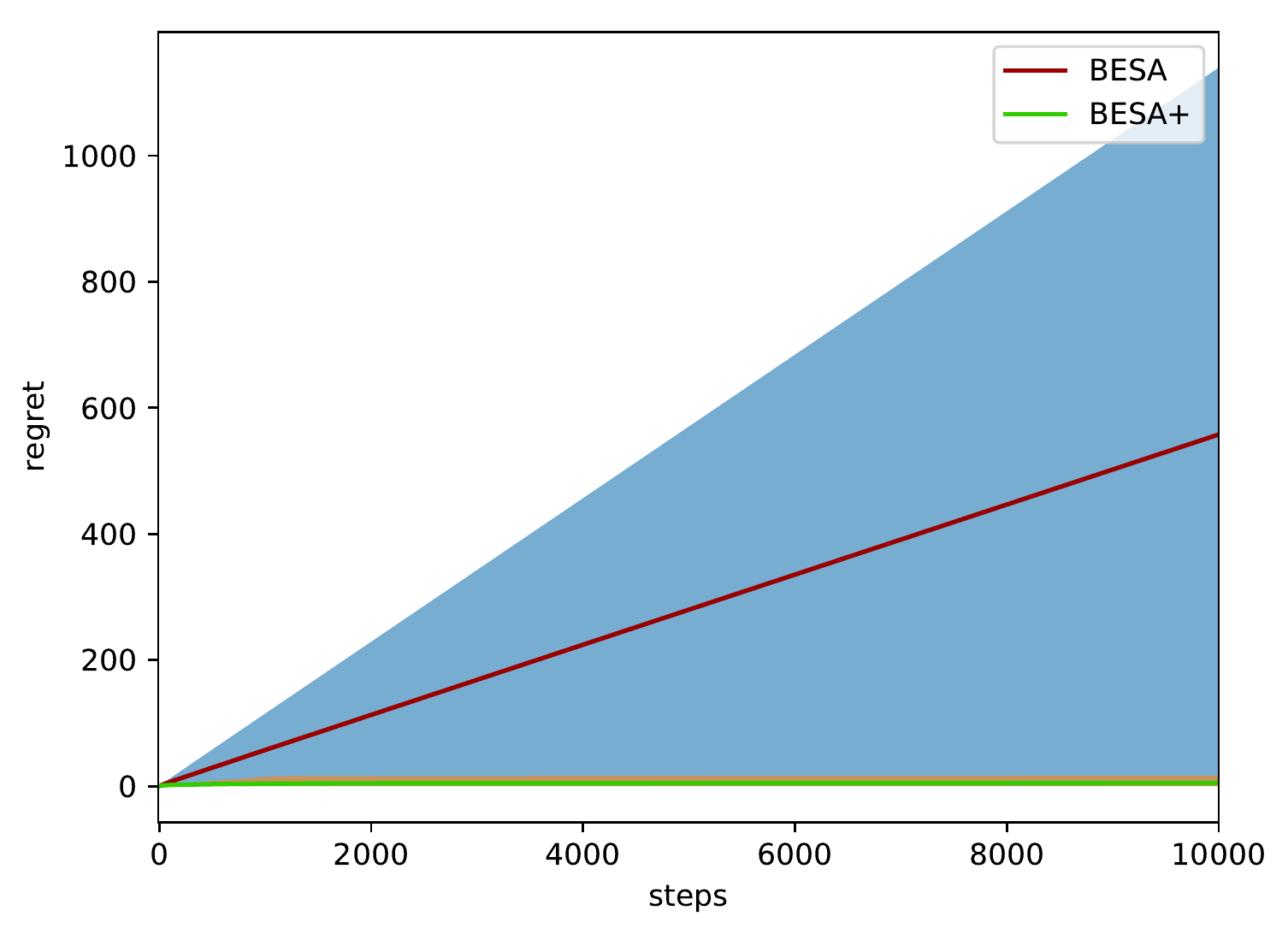}
	\caption{Result of accumulated expected regret for $r=0.2$}\label{Fig:Data2}
\end{figure}

\subsection{Conditional value at risk safety value function}

As discussed in \cite{galichet2013exploration}, in some situations, we need to limit the exploration of risky arms. Examples include financial investment where inverters may tend to choose risk-averse kind of strategy. Using conditional value at risk as a risk measure is one of the approaches to achieve this goal. Informally, conditional value at risk level $\alpha$ is defined as the expected values of the quantiles of reward distribution where the probability of the occurrence of values inside this quantile is less than or equal to $\alpha$. More formally:
\begin{align}\label{cvar}
CVaR_\alpha=\mathbb{E}[X|X<v_\alpha]
\end{align}
where in \eqref{cvar}, $v_\alpha=\argmax_\beta\{\mathbb{P}(X<\beta)\leq\alpha\}$.
To estimate \eqref{cvar}, we have used the estimation measure introduced by \cite{chen2007nonparametric}. This estimation is also employed in \cite{galichet2013exploration} work to derive their MARAB algorithm. Here, we have used this estimation for the Conditional value at risk safety value function which is the regret measure for this problem. Our environment consists of 20 arms where each arm reward distribution is the truncated Gaussian mixture consisting of four Gaussian distribution with equal probability. The reward of arms are restricted to the interval $[0,1]$. To make the environment more complex, the mean and standard deviation of arms are sampled uniformly from the interval $[0,1]$ and $[0.5,1]$ respectively. The experiments are carried out for $\alpha=10\%$. For MARAB algorithm, we have used grid search and set the value $C=1$. The figures 4, 5 depict the results of the run for ten experiments. It is noticeable that in both figures BESA+ has a lower variance in experiments.
\begin{figure}[!htb]
	\centering
	\includegraphics[width=0.7\linewidth]{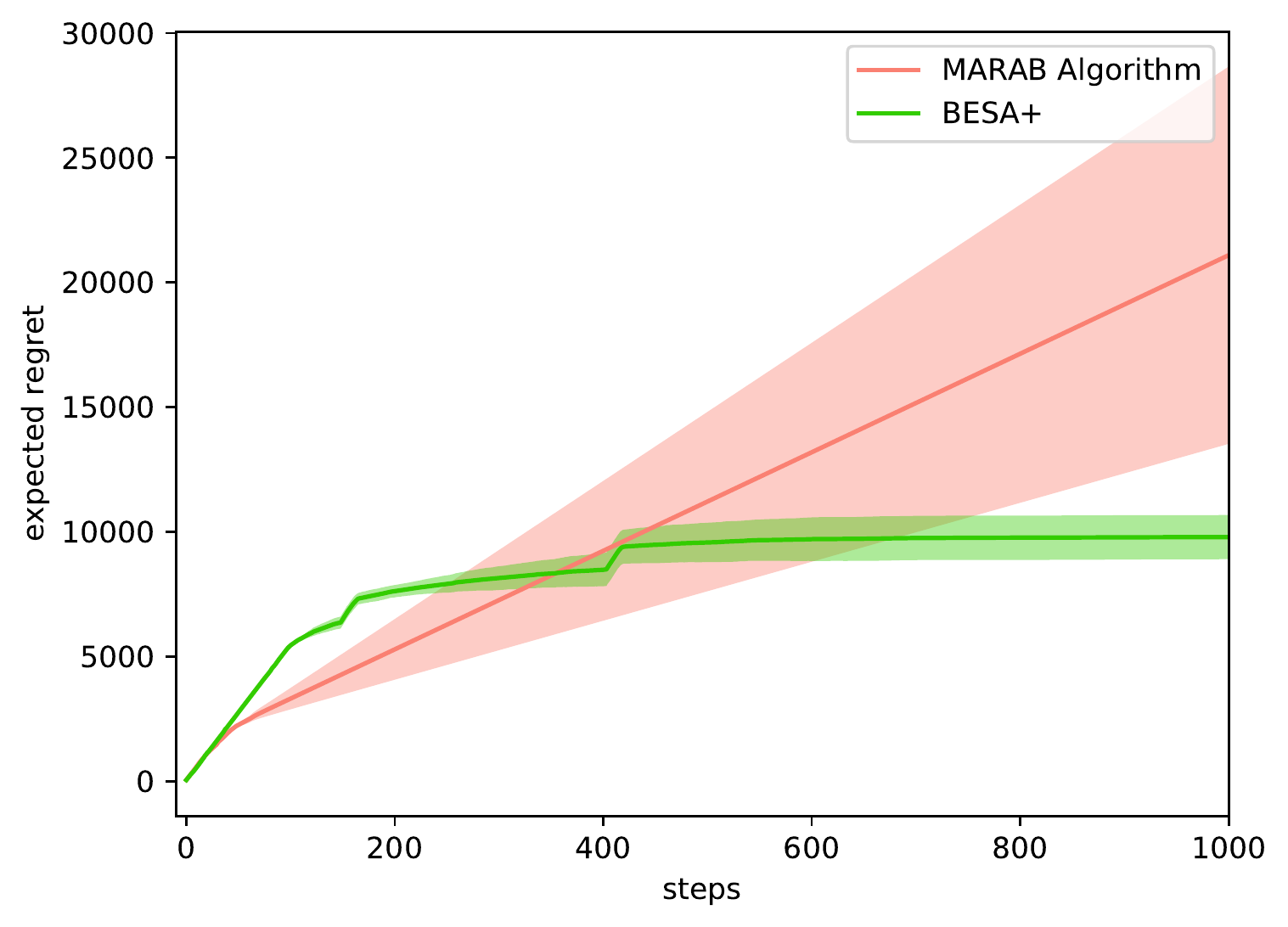}
	\caption{Accumulated regret figure. The safety value function here is conditional value at risk.}
	\label{Fig:Data1} 
	\centering
	\includegraphics[width=0.7\linewidth]{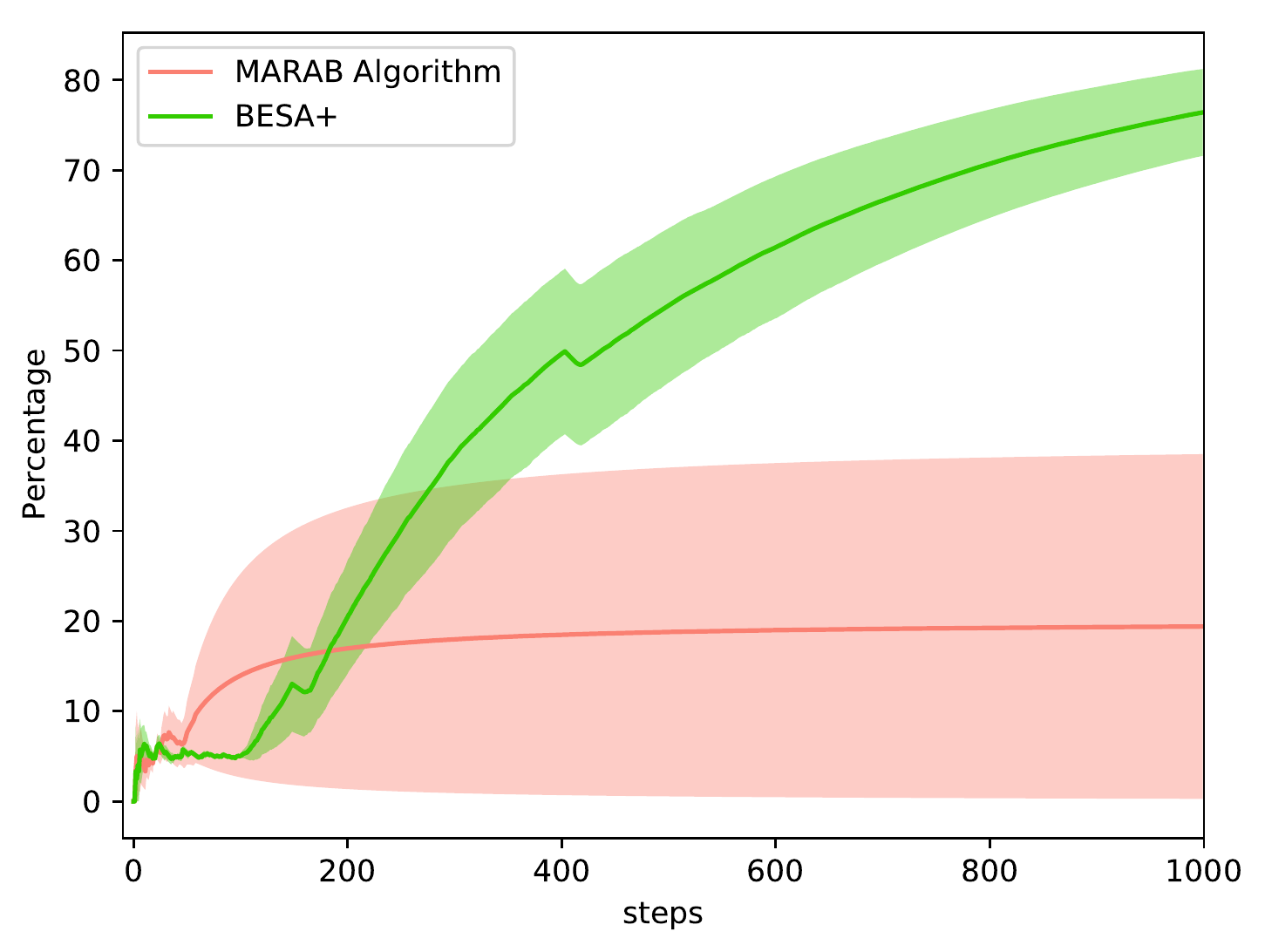}
	\caption{Percentage of optimal arm play figure. The safety value function here is conditional value at risk.}\label{Fig:Data2}
\end{figure}

\subsection{Mean-variance safety value function}

Next, we evaluated the performance of BESA+ with the regret definition provided by \cite{sani2012risk}. Here, we used the same 20 arms Gaussian mixture environment described in the previous section. We evaluated the experiments with $\rho=1$ which is the trade off factor between variance and the mean. The results of this experiment is depicted in figures 6, 7. The hyper-parameters used here for algorithms MV-LCB and ExpExp are based on what \cite{sani2012risk} suggests using. Again, we can see that BESA+ has a relatively small variance over 10 experiments.
\begin{figure}[!htb]
	\centering
	\includegraphics[width=0.7\linewidth]{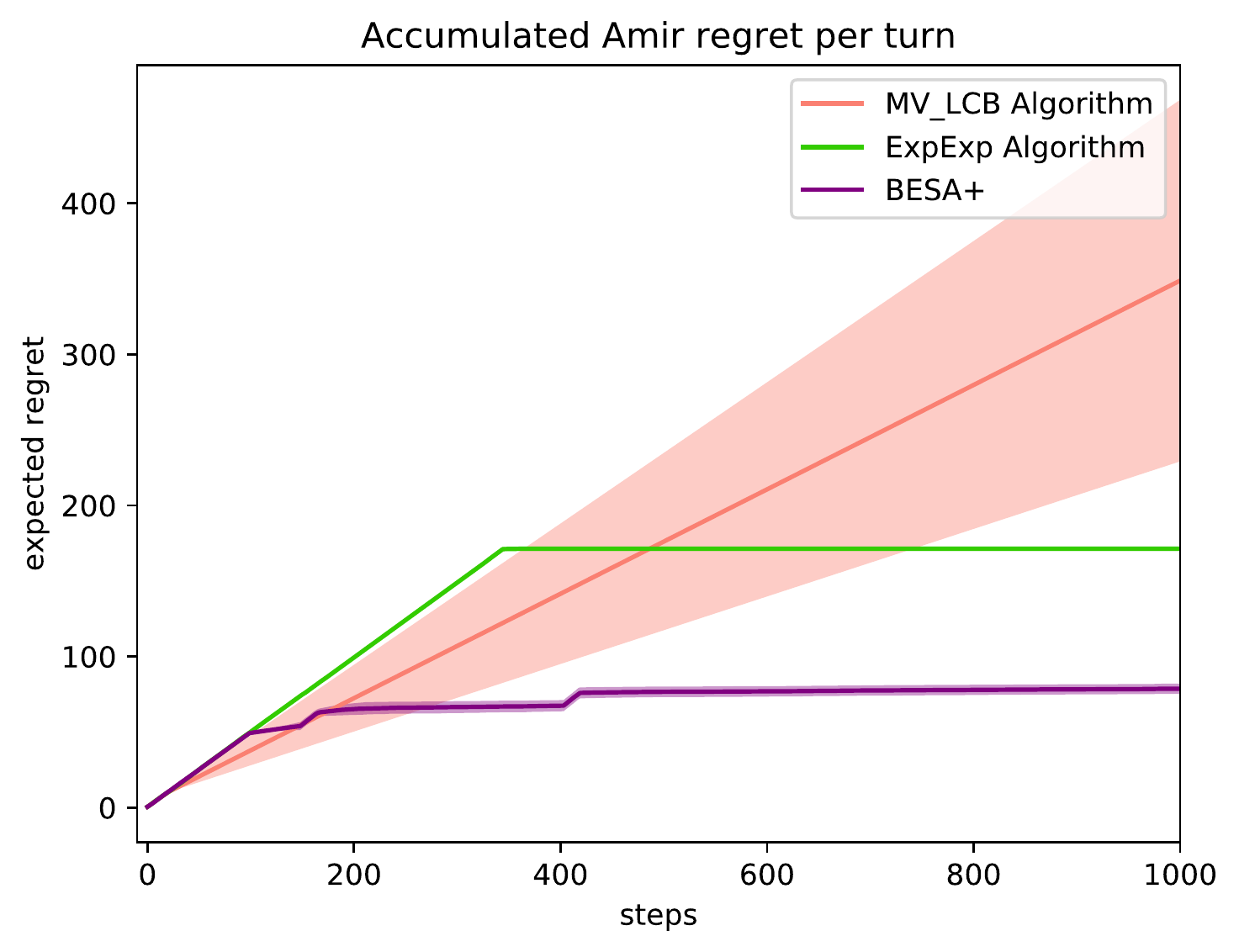}
	\caption{Accumulated regret figure. The safety value function here is mean-variance.}
	\label{Fig:Data1} 
	\centering
	\includegraphics[width=0.7\linewidth]{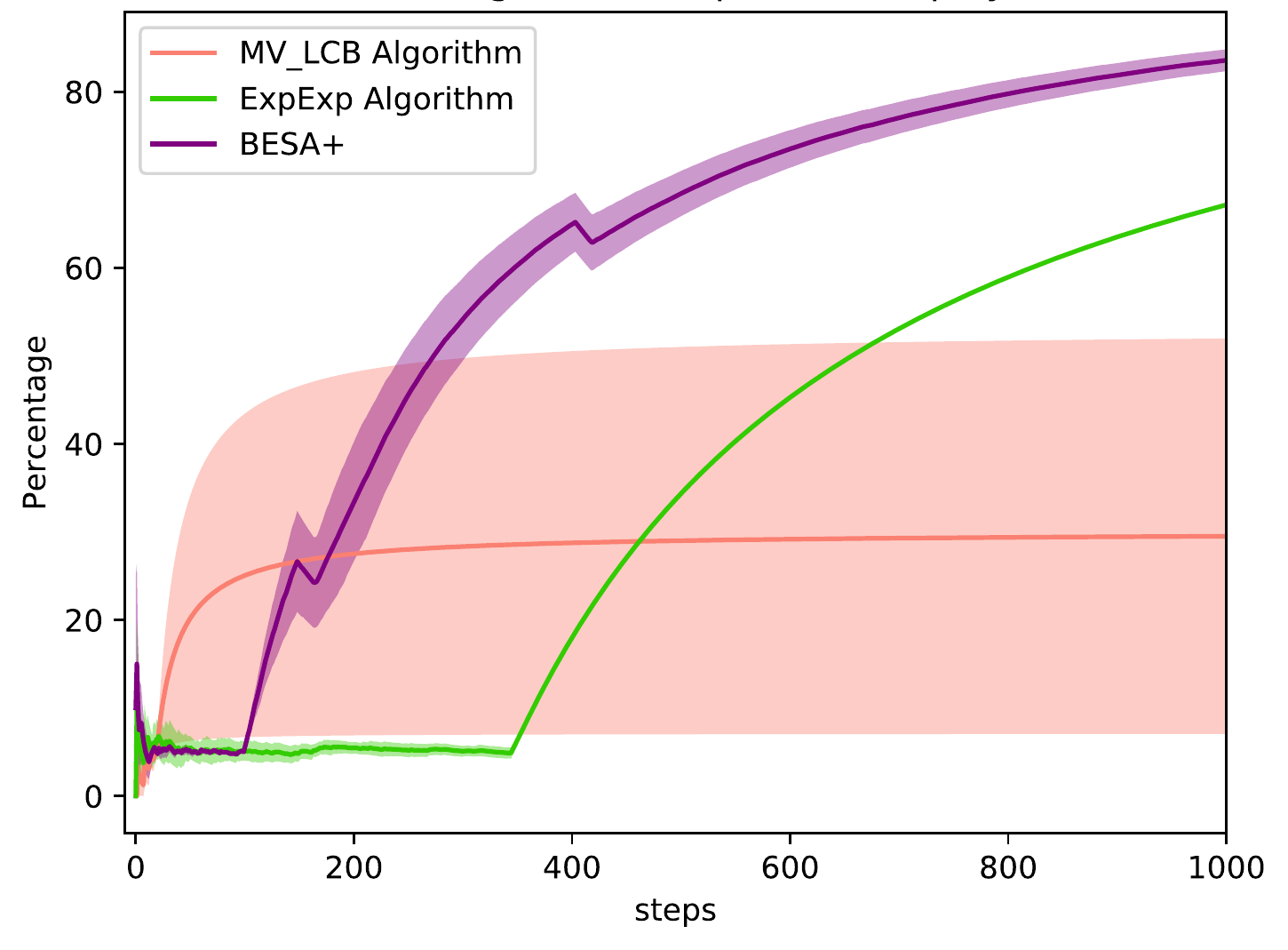}
	\caption{Percentage of optimal arm play figure. The safety value function here is mean-variance.}\label{Fig:Data2}
\end{figure}

\subsection{Real Clinical Trial Dataset }

Finally, we examined the performance of BESA+ against other methods (BESA, UCB1 , Thompson sampling, MV-LCB, and ExpExp) based on a real clinical dataset. This dataset includes the survival times of patients who were suffering from lung cancer \cite{ripley2013package}. Two different kinds of treatments (standard treatment and test treatment) were applied to them and the results are based on the number of days the patient survived after receiving one of the treatments. For the purpose of illustration and simplicity, we assumed non-informative censoring and equal follow-up times in both treatment groups. As the experiment has already been conducted, to apply bandit algorithms, each time a treatment is selected by a bandit algorithm, we sampled uniformly from the recorded results of the patients whom received that selected treatment and used the survival time as the reward signal. Figure 8 shows the distribution of treatment 1 and 2. We categorized the survival time into ten categories (category 1 showing the minimum survival time). It is interesting to notice that while treatment 2 has a higher mean than treatment 1 due to the effect of outliers, it has a higher level of variance compared to treatment 1. From figure 8 it is easy to deduce that treatment 1 has a more consistent behavior than treatment 2 and a higher number of patients who received treatment 2 died early. That is why treatment 1 may be preferred over treatment 2 if we use the safety value function described in Example 1. In this regard, by setting $\rho=1$, treatment 1 has less expected mean-variance regret than treatment 2, and it should be ultimately favored by the learning algorithm. Figure 9  illustrates the performance of different bandit algorithms. It is easy to notice that BESA+ has relatively better performance than all the other ones.

\begin{figure}[!htb]
	\centering
	\includegraphics[width=0.7\linewidth]{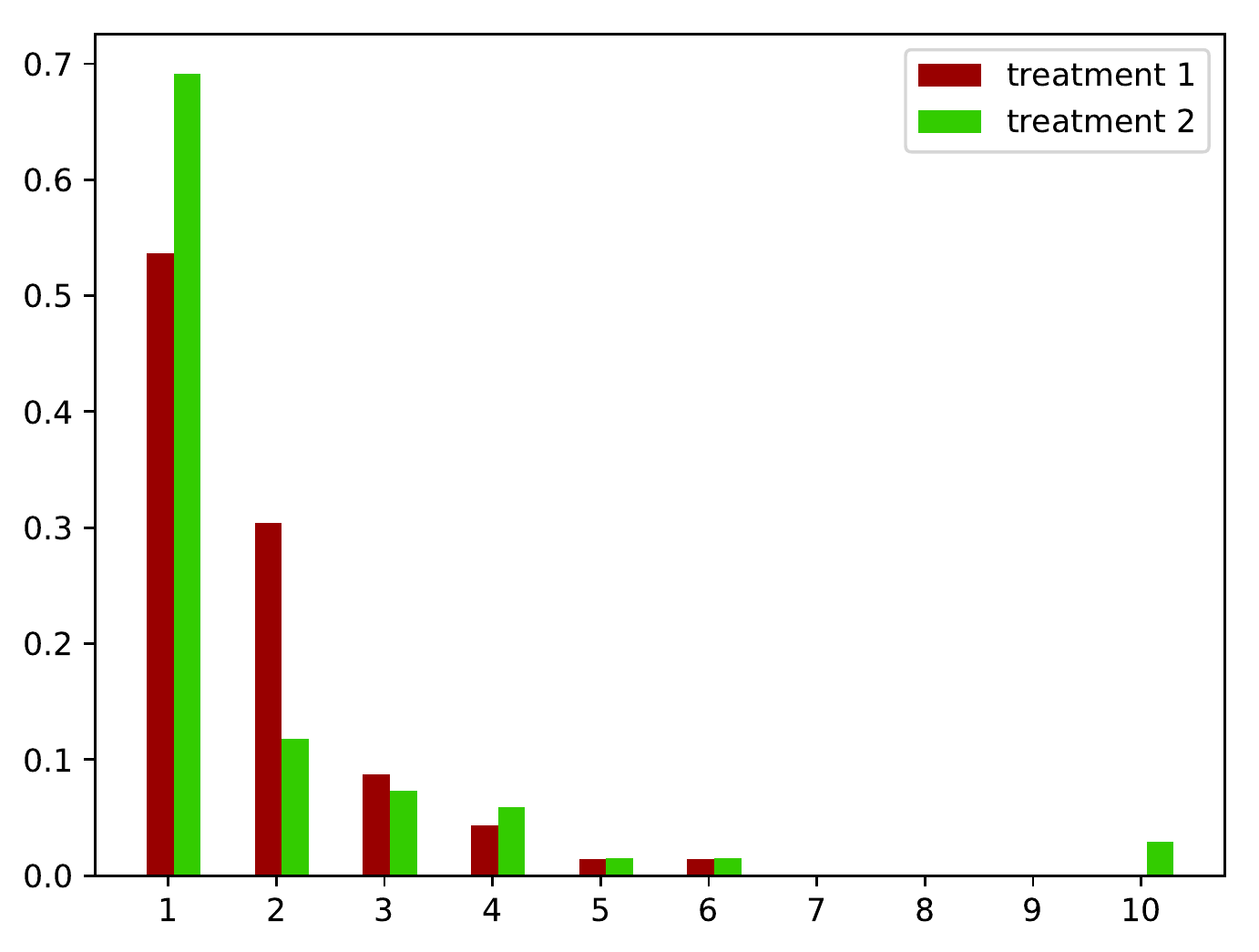}
	\caption{Distribution graph}
	\label{Fig:Data1} 
	\centering
	\includegraphics[width=0.7\linewidth]{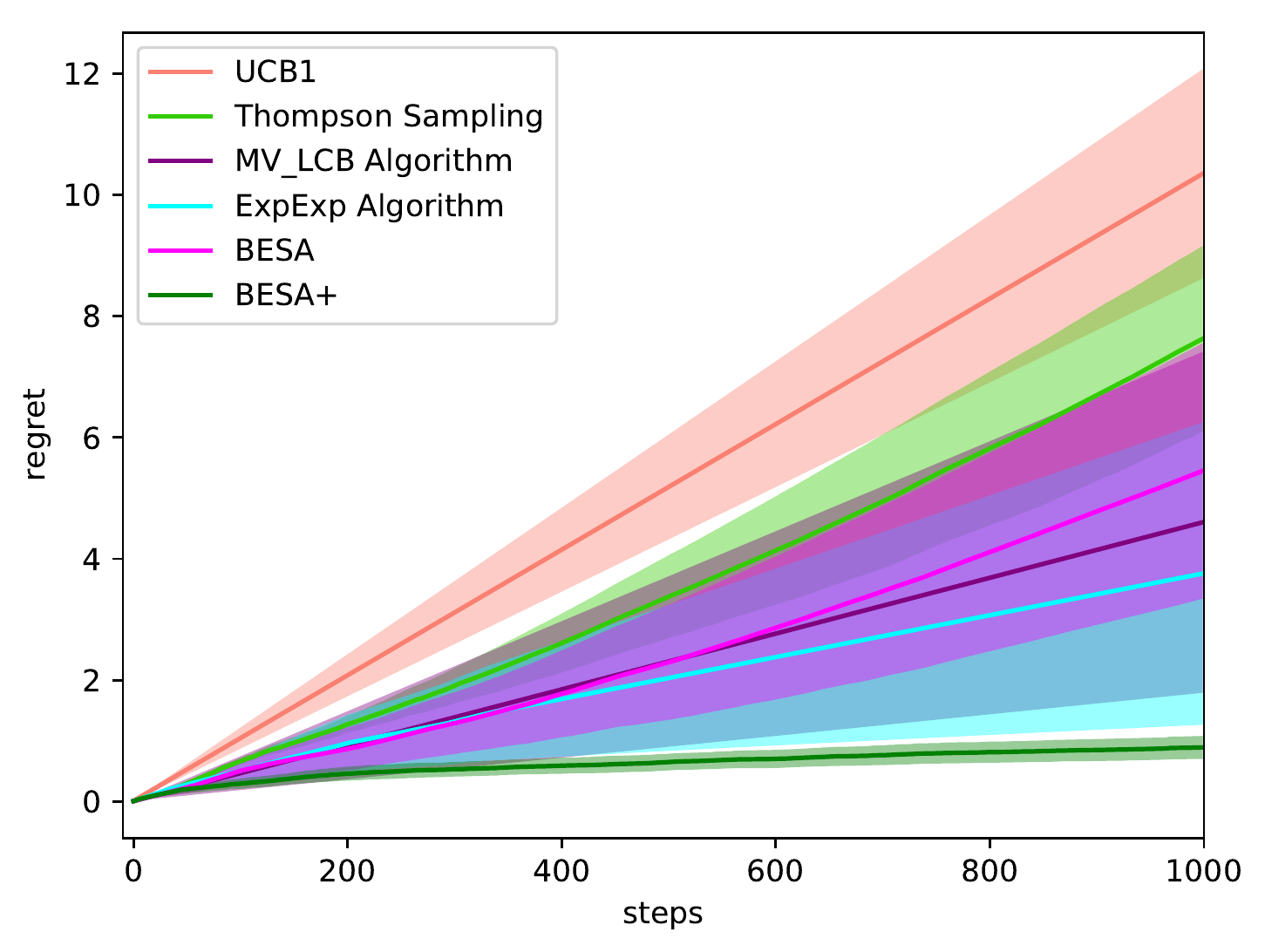}
	\caption{Accumulated consistency-aware regret}
\end{figure}

\subsection{Web Application Simulator }

As discussed earlier, for this project, we have developed a web application simulator for bandit problem where users can create their customized environment and run experiments online. Usually, research works provide limited experiments to testify their method. We tried to overcome this problem by developing this web application where the user can select number of arms and change their reward distribution. Then the web application will send the input to the web-server and show the results to the user by providing regret figures and additional figures describing the way algorithms have chosen arms over time.  This software can be used as a benchmark for future bandit research and it is open sourced for future extension. The link is going to be provided after the review process.
%    \item Consider truncated Gaussian and Beta settings.
%    \item Compare with RA-UCB~\citep{Maillard2013}
%    \item Compare with MV-LCB and ExpExp~\citep{Sani2012}
%\end{itemize}

% !TeX root = IJCAIpaper.tex

\section{Conclusion and future work}

In this paper, we  developed a modular safety-aware regret definition which can be used to define the function of interest as a safety measure. We also modified the BESA algorithm and equipped it with new features to solve  modular safety-aware regret bandit problems. We then computed the asymptotic regret of BESA+ and showed that it can perform like an admissible policy if the safety value function satisfies a mild assumption. Finally, we depicted the performance of BESA+ on the regret definition of previous works and showed that it can have better performance in most cases.

It is still interesting to investigate whether we can find better bounds for BESA+ algorithm with modular safety-aware regret definition. Another interesting path would be to research if we can define similar safety-aware regret definition for broader reinforcement learning problems including MDP environments.

\section{Acknowledgment}

 We would like to thank Audrey Durand for her comments and insight on this project. We also thank department of family medicine of McGill University and CIHR for their generous support during this project. 

\bibliographystyle{named}
\bibliography{aaai_refrence.bib}

\end{document}

\title{Appendix}
\date{}
\maketitle

\appendix

%!TEX root = ./appendix1.tex

\section{Proof of Theorem 0.1}

We consider the two-action (two-armed) setting $\cA = \{\star, a\}$. Recall that $\star := \argmax_{i \in \{\star, a\}} v_i$ is the optimal action and let $\Delta := v_\star - v_a$ denote the suboptimal gap.
We first rewrite the cumulative pseudo-regret up to time $T$ as
\begin{align}
    \kR_T = \Delta \Esp[N_{a,T}].
\end{align}

\paragraph{Notation:} We use the notation $\hat q_{i,t}$ to denote the empirical estimate of the quantity $q$ associated with action $i$ given previous observations.  We also use the notation $\tilde q_{i,t}$ to denote the empirical estimate of the quantity $q$ associated with action $i$ computed on a subsample without replacement of previous observations. Also, $N_{a,T}$ denotes the number of time arm $a$ has been selected by the algorithm up to time $T$.

 Our proof is going to be to some extent similar to \cite{Baransi2014}. We see that the regret can be bounded by controlling the expected number of suboptimal plays, which we decompose as follows for any $u > 0$:
\begin{align} \label{regret_expectation}
    \Esp[N_{a,T}]
    &= \Esp \left[ \sum_{t=1}^T \Big( \indic{\tilde v_{\star,t} \leq \tilde v_{a,t} \cap N_{\star,t} > u} \cup \indic{\tilde v_{\star,t} \leq \tilde v_{a,t} \cap N_{\star,t} \leq u} \Big) \right] \nonumber \\
    &\leq \Esp \left[ \sum_{t=1}^T \Big( \indic{\tilde v_{\star,t} \leq \tilde v_{a,t} \cap N_{\star,t} > u} \cup \indic{N_{\star,t} \leq u} \Big) \right] \nonumber \\
    &= \sum_{t=1}^T \Pr \left[ \tilde v_{\star,t} \leq \tilde v_{a,t} \cap N_{\star,t} > u \right] + \sum_{t=1}^T \Pr\left[ N_{\star,t} \leq u \right].
\end{align}

As it can be seen, \eqref{regret_expectation} consists of two terms where one term depends on  $N_{\star,t} > u $ and the other one depends on $N_{\star,t} \leq u$. In this proof, we will show the probability of the first term and then prove that the probability of $N_{\star,t} \leq u$ is controllable.

\subsection{Number of play of suboptimal arm}

Consider a sequence of size $N_{a,t}$, $X_{1:N_{a,t}}^a:=\{X_1^a, ..., X_{N_{a,t}}^a\}$, consisting of observations from arm $a$ up to time $t$.
% Assume $\hat{\mu}$ is the empirical mean of this sequence. In this regard,  consider the following sequence as the transformed version of the initial  sequence $X_{1:N_{a,t}}^a$:
% \begin{align}\label{magic}
% V_{1:N_{a,t}}^a:=\{X_1^a -\frac{N_{a,t}}{N_{a,t}-1}\rho (X_1^a -\hat{\mu})^2, ..., X_{N_{a,t}}^a -\frac{N_{a,t}}{N_{a,t}-1}\rho (X_{N_{a,t}}^a -\hat{\mu})^2\}
% \end{align}
% In this regard, the value calculated in the generalized BESA algorithm for arm $a$ is
% \begin{align}
% \hat{v}_{a,N_{a,t}}=\frac{1}{N_{a,t}}\sum_{i=1}^{N_{a,t}} X_i^a -\frac{N_{a,t}}{N_{a,t}-1}\rho (X_i^a -\hat{\mu})^2=\frac{1}{N_{a,t}}\sum_{v_i \in V_{1:N_{a,t}}^a} v_i
% \end{align}
 % which is an unbiased estimate of the $v_a=\mu-\rho\sigma^2_a$. Thus, we can apply McDiarmid inequality \citep{tolstikhin2017concentration} for the variables of the sequence \eqref{magic} and have an upper bound on its probability of error from the true value $v^a$:
Using Lemma~\ref{lem:mcdiarmid}, we have that
\begin{align}
    \label{eqn:bound:sample}
    \Pr\left[ \hat v^a_{N_{a,t}} \leq v^a - \gamma \sqrt{\frac{\log(1/\delta)}{2N_{a,t}}} \right] \leq \delta.
\end{align}
Similarly, using Lemma~\ref{lem:mcdiarmid:without_replacement} for the subsample of $N = \min(N_{a,t}, N_{\star,t})$ without replacement, we have that
\begin{align}
    \label{eqn:bound:subsample}
    \Pr\left[ \tilde v_N^a \leq \hat v_{N_{a,t}}^a - \gamma \sqrt{\frac{\min(N, N_{a,t}-N) \log(1/\delta)}{2N^2}} \right] \leq \delta.
\end{align}

By definition of the BESA algorithm, it holds that $a_t = a$ implies
\begin{align*}
    \tilde v_{\star,t} &\leq \tilde v_{a,t} \\
    \tilde v_{\star,t} + \hat v_{\star,t} - \hat v_{\star,t} + v_\star - v_\star &\leq \tilde v_{a,t} + \hat v_{a,t} - \hat v_{a,t} + v_a - v_a \\
    \Delta &\leq (\tilde v_{a,t} - \hat v_{a,t}) + (\hat v_{a,t} - v_a) + (\hat v_{\star,t} - \tilde v_{\star,t}) + (v_{\star} - \hat v_{\star,t}).
\end{align*}

Now, using Eq.~\ref{eqn:bound:sample} and~\ref{eqn:bound:subsample}, we have with probability higher than $1-2\delta$ that
\begin{align}
    \label{eqn:delta}
    \Delta &\leq
    \gamma \sqrt{\frac{\min(N, N_{a,t}-N) \log(1/\delta)}{2N^2}} +
    \gamma \sqrt{\frac{\log(1/\delta)}{2N_{a,t}}} \nonumber \\
    &+ \gamma \sqrt{\frac{\min(N, N_{\star,t}-N) \log(1/\delta)}{2N^2}} +
    \gamma \sqrt{\frac{\log(1/\delta)}{2N_{\star,t}}} \nonumber \\ 
    &\leq
    2\gamma \sqrt{\frac{ \log(1/\delta)}{2N}} +
    \gamma \sqrt{\frac{\log(1/\delta)}{2N_{a,t}}}  +\gamma \sqrt{\frac{\log(1/\delta)}{2N_{\star,t}}}    
\end{align}

We have two states here. Whether $N_{a,t}<N_{\star,t}$ or $N_{\star,t} \leq N_{a,t}$ Let us first consider the situation where $N_{a,t}<N_{\star,t}$, such that $N = N_{a,t}$. Under this circumstance, Eq.~\ref{eqn:delta}:

\begin{align} \label{delta_1}
    \Delta &\leq 4\gamma \sqrt{\frac{\log(1/\delta)}{2N_{a,t}}} 
    \end{align}
By putting $\delta=\frac{1}{t}$, we can deduce from \eqref{delta_1} that:
\begin{align}\label{delta_2}
N_{a,t} &\leq \frac{16\gamma^2} {\Delta^2} \log(t)
\end{align}

In other words, under this circumstance we have $N_{a,t}=O(\log(t))$ which illustrates that we can control the number of suboptimal arm play when $N_{a,t}<N_{\star,t}$.

Now, by following similar steps for the other case when $N_{a,t}\geq N_{\star,t}$, we have  :
\begin{align}\label{delta_2}
N_{\star,t} &\leq \frac{16\gamma^2} {\Delta^2} \log(t)
\end{align}

As a result, if $N_{\star,t} $ is played played more than $u_t=\frac{16\gamma^2} {\Delta^2} \log(t)$, the algorithm chooses suboptimal arm $O(\log(t))$ with probability $1-\frac{2}{t}$. 

In the next subsection, we will illustrate how we can upper-bound the probability of $N_{\star,t} <u_t$ and based on that we will compute the final expected regret of BESA+ algorithm.

\subsection{Probability of not selecting optimal arm frequently}

Following the idea of \cite{Baransi2014}, let $t_j$ denote the time step at which action $\star$ is selected for the $j$-th time, with $t_0 = 0$, and $\tau_j = t_{j+1} - t_j - 1$ being the number of time steps between the $j$-th and the $(j+1)$-th play of action $\star$. Note that $\tau_0 \leq 2$ since all actions are played once during the initialization step. Assume that $t\geq 3u_t$ we then have:
\begin{align}\label{event_1}
    \Pr [N_{\star,t} \leq u_t]
    &= \Pr[N_{a,t} > t - u_t] \\\nonumber
    &\leq   \Pr[\exists j \in \{1,\dots,\floor{u_t} \}: \tau_j > t /u_t-1] \\\nonumber
    &\leq  \sum_{j=1}^{\floor{u_t}} \Pr[\tau_j > t /u_t-1] \\
\end{align}

Now if we show that the the probability of the event $E_j=\{\tau_j > t /u_t-1\}$ is ignorable, we can hopefully deduce that it is unlikely that the BESA+ algorithm selects the optimal arm less than $u_t$ time and probably we  can control the final expected regret. In this regard, it is easy to see that:
\begin{align}\label{event_2}
    \Pr [E_j]
%    &= \Pr[N_{a,t} > t - u_t] \\\nonumber
    &\leq   \Pr[\forall j \in \{t_j,\dots,t_j+\floor{t/u_t }-1\}: a_s=a] \\\nonumber
\end{align}

%Now if we split the time interval $M_{j,t}=[t_j,t_j+\ceil{t/u_t-1}]$ to two sets:
%\begin{align*}
%M_{j,t}^1&=[t_j,t_j+\ceil{\frac{t/u_t-1}{2}}]\\
%M_{j,t}^2&=[t_j+\ceil{\frac{t/u_t-1}{2}},t_j+t/u_t-1]\\
%\end{align*} 
%As action $a$ is chosen repeatedly during $M_{j,t}^1$, for all time step $s$ in $M_{j,t}^2$ we have:
%\begin{align}
%N_{a,s}=s-j\geq t_j+\ceil{\frac{t/u_t-1}{2}}-j \geq \ceil{\frac{t/u_t-1}{2}}
%\end{align}
%In other words, when  $t\geq u_t(1+2u_t)$, for all $s \in M_{j,t}^2$ we have $N_{a,s}\geq u_t\geq j$. Accordingly, as mentioned in \citep{Baransi2014}, we need to define the quantity $n_{t,j}=\max \{\ceil{\frac{t/u_t-1}{2}},j\}$.
%
%Now, using the assumption that $t\geq u_t(1+2u_t)$ and previous arguments, we can find the probability of event $E_j$: 
%\begin{align}\nonumber \label{master_lemma}
%\Pr[E_j] &\leq \Pr[\forall s \in M_{j,t}: a_s=a]\\\nonumber
%&\leq \Pr[\forall s \in  M_{j,t}^2: a_s =a \cap N_{a,s} \geq n_{t,j}]\\\nonumber
%&\leq \Pr[\forall s \in  M_{j,t}^2: \hat{v}_{a,s}\left(I_s(N_{a,s},j)\right)>\hat{v}_{\star,s}\cap N_{a,s} \geq n_{t,j}]\\
%&\leq \Pr[\forall s \in  M_{j,t}^2: \hat{v}_{a,s}\left(I_s(s-j,j)\right)>\hat{v}_{\star,s} ]
%\end{align}
%where the last line of \eqref{master_lemma} stems from the fact that we have $N_{a,s}=s-j$ and as the number of times we have chosen arm $a$ is more than arm $\star$, we will subsample from arm $a$. We want to show that the probability of \eqref{master_lemma} is small when $v_a<v_\star$. 

Note that as in BESA+ algorithm, given the record (reward history) of an arm $a$, the sub-sampling process of arm $a$ is independent of the previous ones. Define the interval $M_{j,t}=[t_j,t_j+\floor{t/u_t}-1]$. 

By applying McDiarmid’s inequality (lemma B.1 and B.2) and assuming $t>c_{\Delta,\gamma}$ where $c_{\Delta,\gamma}>\exp(\frac{4\gamma^2\log(2)}{(v_a-v_\star)^2})$, we have:

\begin{align}\nonumber \label{master_lemma_2}
\Pr[E_j] &\leq \Pr[\forall s \in M_{j,t}: a_s=a]\\\nonumber
%&\leq \Pr[\forall s \in  M_{j,t}^2: \tilde{v}_{a,s}\left(I_s(s-j,j)\right)>\tilde{v}_{\star,s} ]\\\nonumber
&=\prod_{s=t_j}^{t_j+\floor{t/u_t}-1}\Pr[\tilde{v}_{a,s}\left(\cX_{a,s}(I_a(s-j,j))\right)>\tilde{v}_{\star,s} ]\\\nonumber
&=\prod_{s=t_j}^{t_j+\floor{t/u_t}-1}\Pr[\tilde{v}_{a,s}\left(\cX_{a,s}(I_a(s-j,j))\right)-\tilde{v}_{\star,s}> 0]\\\nonumber
&=\prod_{s=t_j}^{t_j+\floor{t/u_t}-1}\Pr[(\tilde{v}_{a,s}(\cX_{a,s}(I_a(s-j,j)))-v_a)-(\tilde{v}_{\star,s}-v_\star)> v_\star-v_a]\\\nonumber
&=\prod_{s=t_j}^{t_j+\floor{t/u_t}-1}\Pr\left[~\left[(\tilde{v}_{a,s}(\cX_{a,s}(I_a(s-j,j)))-\tilde{v}_{\star,s})-(\hat{v}_{a,s}-\hat{v}_{\star,s})\right]+\left[(\hat{v}_{a,s}-\hat{v}_{\star,s})-(v_a-v_\star)\right]>v_\star-v_a~\right]\\\nonumber
&\leq\prod_{s=t_j}^{t_j+\floor{t/u_t}-1}\left(\Pr\left[(\tilde{v}_{a,s}(\cX_{a,s}(I_a(s-j,j)))-\tilde{v}_{\star,s})-(\hat{v}_{a,s}-\hat{v}_{\star,s})>\frac{1}{2}(v_{\star}-v_{a})~\right]+ \Pr\left[(\hat{v}_{a,s}-\hat{v}_{\star,s})-(v_a-v_\star) >\frac{1}{2}(v_\star-v_a)~\right]\right) \\\nonumber
&\leq \prod_{s=t_j}^{t_j+\floor{t/u_t}-1} \exp \left(\frac{-j(v_a-v_\star)^2}{2\gamma^2}+\log(2)\right)~~~~~~&&\\\nonumber
&\leq \prod_{s=t_j}^{t_j+\floor{t/u_t}-1} \exp \left(\frac{-j(v_a-v_\star)^2}{4\gamma^2}
\right) ~~~~~~~~~~~~~~~~~~~~~~~~~~~~~~~~~~~~~~~~~~~~~~~~~~~~~ \text{since we have $j \geq \log(t)$}
\\
&= \exp \left(\frac{-j(\floor{t/u_t})(v_a-v_\star)^2}{4\gamma^2}
\right)
\end{align}

In order to make computations easier we define new notations here. We know that $u_t=m\log(t)$ where $m$ is a constant dependent on $v_a-v_\star$. Define $\omega=\frac{(v_a-v_\star)^2}{4\gamma^2}$. Also define $\kappa=\frac{\omega}{m}$.

Now, we have: 

\begin{align} \label{Sum_first}
\sum_{t=1}^T\Pr[N_{\star,t} \leq u_t]&\leq c_{\Delta,\gamma}+\sum_{t\geq c_{\Delta,\gamma}}^T \sum_{j=1}^{\floor{u_t}} \exp \left(-j\floor{\frac{t}{m\log(t)}}\omega \right) \nonumber \\
&\leq c_{\Delta,\gamma} +\sum_{t\geq c_{\Delta,\gamma}}^T e^{-\floor{\frac{t}{m\log(t)}}\omega} \frac{1-e^{-t\omega}}{1-e^{-\floor{\frac{t}{m\log(t)}}\omega}}  \nonumber \\
&\leq c_{\Delta,\gamma} +\sum_{t\geq c_{\Delta,\gamma}}^T Ce^{-\frac{\kappa t}{\log(t)}} (1-e^{-t\omega}) \nonumber \\
&\leq c_{\Delta,\gamma}+ \beta(T,\omega,C)
\end{align}

Where $C=\frac{e^\omega}{1-e^{-3\omega}}$ and $$\beta(T,\omega,C)=\sum_{t\geq c_{\Delta,\gamma}}^{T} Ce^{-\frac{\kappa t}{\log (t)}} (1-e^{-t\omega}) $$ It turns out that based on Lemma 2, $\beta(T,\omega,C)$ can be bounded by a harmonic series when $c_{\Delta,\gamma}= \max\{ \frac{\log(T)}{\kappa},\exp(\frac{4\gamma^2\log(2)}{(v_a-v_\star)^2}),3u_t\}$. As a result, \eqref{Sum_first} is upper bounded by $O(\log(T))$. Before digging into further analysis, let's consider the following Lemma:

\textbf{Lemma 1}: for all $x> 1$ we have:
$$\log(x)+\frac{1}{\log(x)} > \log(x+1)$$

\textbf{Proof:} Using the mean value theorem, we know that there exists $x<y<x+1$ such that:
$$\log(x+1)-\log(x)=\frac{1}{y}$$
As a result we have:
$$\log(x+1)-\log(x)<\frac{1}{x}$$
Now, it is sufficient to show that:
$$\frac{1}{\log(x)}\geq \frac{1}{x}$$ 
which holds for any $x>1$. \qedsymbol

\textbf{Lemma 2}: $\beta(T,\omega,C)$ can be upper bounded by a harmonic series when $c_{\Delta,\gamma}= \max\{ \frac{\log(T)}{\kappa},\exp(\frac{4\gamma^2\log(2)}{(v_a-v_\star)^2}),3u_t\}$ and we have:

\begin{align}
\label{beta_function}
\beta(T,\omega,C) =O(log(T))
\end{align}

\indent \textbf{Proof:}
We can write the $\beta(t,\omega,C)$ as follows:
\begin{align} \label{Sum_second}
\beta(t,\omega,C)&=\sum_{t\geq c_{\Delta,\gamma}}^T Ce^{-\frac{\kappa t}{\log(t)}} (1-e^{-t\omega})\nonumber\\
&\leq \sum_{t\geq  \log(T)/\kappa}^{\frac{(log(T))^2}{\kappa}} C e^{-\frac{\kappa t}{\log (t)}} (1-e^{-t\omega})+\sum_{t\geq \frac{(log(T))^2}{\kappa}}^{T} Ce^{-\frac{\kappa t}{\log (t)}} (1-e^{-t\omega})
\end{align}

Now we will study each sum in the RHS of  \eqref{Sum_second} separately. We need to show that the first sum is bounded by $O(\log(T))$. If we expand the first sum, we notice the following:

\begin{align}
\sum_{t\geq  \log(T)/\kappa}^{\frac{(log(T))^2}{\kappa}} C e^{-\frac{\kappa t}{\log (t)}} (1-e^{-t\omega})&\leq \sum_{t\geq  \log(T)/\kappa}^{\frac{(log(T))^2}{\kappa}} C e^{-\frac{\kappa t}{\log (T)}} \nonumber\\
 &\leq C\bigg[\frac{1}{e}+\frac{1}{e^{(1+\kappa/log(T))}} + \dots \nonumber\\
 &\;\;\;\;\;\;+ \frac{1}{e^2} + \frac{1}{e^{(2+\kappa/log(T))}} + \dots  \nonumber\\
 &\;\;\;\;\;\;\;\;\vdots \nonumber\\
 &\;\;\;\;\;\;+ \frac{1}{e^{log(T)}} \bigg]\nonumber\\
& \leq \frac{C}{\kappa}\sum_{i=1}^{\log (T)} \frac{\log(T)}{e^i}=\frac{C}{\kappa}\times\frac{\log(T)}{e}\times\frac{1-\frac{1}{T}}{1-\frac{1}{e}}\nonumber\\
&=O(\log(T))
\end{align}

For the second sum in  \eqref{Sum_second}, we note that we can actually upper bound it by harmonic series by using Lemma 1 as follows:

\begin{align}
\sum_{t\geq \frac{(log(T))^2}{\kappa}}^{T} Ce^{-\frac{\kappa t}{\log (t)}} (1-e^{-t\omega}) &\leq \sum_{t\geq \frac{(log(T))^2}{\kappa}}^{T} Ce^{-\frac{\kappa t}{\log (T)}} \nonumber\\
&\leq C \bigg[ \frac{1}{T} +\frac{1}{e^{\log(T) +\kappa/\log(T)  }}+\dots \nonumber\\
&\;\;\;\;\;\;+ \frac{1}{T+1} +\frac{1}{e^{\log(T+1) +\kappa/\log(T+1)  }}+\dots \nonumber\\
&\;\;\;\;\;\;\;\;\vdots \nonumber\\
&\;\;\;\;\;\;+\frac{1}{2T} \bigg]\nonumber\\
&\leq \frac{C}{\kappa} \sum_{i=1}^{2T} \frac{1}{i} \leq \frac{C}{\kappa} (\log(2T)+1)\nonumber\\
&=O(\log(T))
\end{align}\qedsymbol
\subsection{Final regret bound}

Back to our initial argument, and using the last two subsection upper-bound we have:

\begin{align} \label{final_regret}
\kR_T &\leq \Delta \sum_{t=1}^T \Pr \left[ \tilde v_{\star,t} \leq \tilde v_{a,t} \cap N_{\star,t} > u_t \right] + \Delta  \sum_{t=1}^T \Pr\left[ N_{\star,t} \leq u_t \right]
\nonumber\\
&\leq \Delta \sum_{t=1}^T \frac{2}{t} +\Delta u_T +\Delta c_{\Delta,\gamma} + \Delta \beta(T,\omega,C)+\Delta g_{T}\nonumber\\
&=  \zeta_{\Delta, \gamma}\log(T) + \theta_{\Delta, \gamma}
\nonumber\\&=O(\log T)
\end{align}

In \eqref{final_regret}, $g_{T}$ is the correction term in BESA+ algorithm which accounts for the minimum number of times we are supposed to pull suboptimal arm up to time $T$ which is bounded by $\log (T)$. $\zeta_{\Delta, \gamma}, \theta_{\Delta, \gamma}$ are constatnts and they are functions of $\Delta, \gamma$.
%Where the last inequality in \eqref{master_lemma_2} is obtained via applying Hoefding lemma ($v_\star-v_a>0$)\citep{Bardenet2015}. In consequence, As $s\rightarrow \infty$, the probability of $\Pr[E_j]$ converges to zero when $t\geq u_t(1+2u_t)$. 

%Thus we can conclude that there exists a time point $t\geq u_t(1+2u_t)$, where after this point the probability of $N_\star,a<u_t$ starts converging to zero. 

\qedsymbol

\section{Tools}

\begin{lemma}[McDiarmid's inequality \cite{tolstikhin2017concentration}]
    \label{lem:mcdiarmid}
    Let $X = X_1,\dots , X_n$ be $n$ independent random variables taking values from some space $\cX$, and assume a function $f: \cX^n \mapsto \Real$ that satisfies the following boundedness condition (bounded differences):
    \begin{align*}
        \sup_{x_1, \dots, x_n, \hat x_i} |f(x_1, x_2, \dots, x_i, \dots, x_n) - f(x_1, x_2, \dots, \hat x_i, \dots, x_n)| \leq c
    \end{align*}
    for all $i\in\{ 1, \dots, n \}$. Then for any $\epsilon > 0$, we have
    \begin{align*}
        \Pr[f(X_1, \dots, X_n) - \Esp[f(X_1, \dots, X_n)] \geq \epsilon] \leq \exp \bigg( - \frac{2\epsilon^2}{nc^2} \bigg).
    \end{align*}
\end{lemma}

Consider $v = \mu - \rho \sigma^2$ and
\begin{align*}
    \hat v(x_1, \dots, x_n)
    &= \frac{1}{n} \sum_{i=1}^n x_i - \rho \frac{1}{n} \sum_{i=1}^n (x_i - \hat \mu)^2 \\
    &= \frac{1}{n} \sum_{i=1}^n x_i - \rho \frac{1}{n^2} \sum_{i=1}^n \sum_{j=1}^n \frac{(x_i - x_j)^2}{2}.
\end{align*}
Then, for $x_i \in [0, 1]$ we have
\begin{align*}
    |\hat v(x_1, x_2, \dots, x_i, \dots x_n) - \hat v(x_1, x_2, \dots, \hat x_i, \dots x_n)| \leq \frac{1}{n} + \rho \frac{1}{2n}.
\end{align*}
Using Lemma~\ref{lem:mcdiarmid} with $c = \frac{1}{n}(1 + \rho/2)$, we have
\begin{align*}
    \Pr[\hat v_n  - v] \geq \epsilon] \leq \exp \bigg( - \frac{2n\epsilon^2}{(1 + \rho/2)^2} \bigg).
\end{align*}

\begin{lemma}[McDiarmid's-like inequality for subsampling without replacement~\cite{el2009transductive}]
    \label{lem:mcdiarmid:without_replacement}
    Let $X = X_1,\dots , X_n$ be $n$ independent random variables taking values from some space $\cX$, and let $Y = {Y_1, \dots, Y_m}$ be $m \leq n$ variables subsampled without replacement from $X$. Also assume a function $f: \cX^n \mapsto \Real$ that satisfies the following boundedness condition (bounded differences):
    \begin{align*}
        \sup_{y_1, \dots, y_m, \hat y_i} |f(y_1, y_2, \dots, y_i, \dots, y_m) - f(y_1, y_2, \dots, \hat y_i, \dots, y_m)| \leq c
    \end{align*}
    for all $i\in\{ 1, \dots, m \}$. Then for any $\epsilon > 0$, we have
    \begin{align*}
        \Pr[f(Y_1, \dots, Y_m) - \Esp[f(Y_1, \dots, Y_m)] \geq \epsilon] \leq \exp \bigg( - \frac{2\epsilon^2}{\min(m, n-m) c^2} \bigg).
    \end{align*}
\end{lemma}
Following a similar logic as previous with Lemma~\ref{lem:mcdiarmid:without_replacement}, with $c = \frac{1}{m}(1 + \rho/2)$, we have
\begin{align*}
    \Pr[\tilde v_{m,n} - \hat v_n] \geq \epsilon] \leq \exp \bigg( - \frac{2m^2\epsilon^2}{\min(m, n-m) (1 + \rho/2)^2} \bigg).
\end{align*}

\bibliographystyle{plain}
\bibliography{app}